\documentclass[journal]{IEEEtran}

\usepackage[style=ieee,natbib=true]{biblatex}
\usepackage{hyperref}
\usepackage{url}

\usepackage{mystyle}

\title{Geodesics of Dynamic Graphs for Regime Change Detection}
\author{William Cappelletti, \thanks{W.\ Cappelletti, and P.\ Frossard are with LTS4, EPFL, 1015 Lausanne, Switzerland (e-mail: \{william.cappelletti, pascal.frossard\}@epfl.ch).}%
\'Etienne Voutaz, \thanks{\'E.\ Voutaz is with Cyber-Defence Campus, Armasuisse, 3604 Thun, Switzerland (email: etienne.voutaz@armasuisse.ch).}%
\thanks{The work of W.C.\ is partially supported by Armasuisse.}%
and Pascal Frossard
}

\begin{document}

\maketitle

\begin{abstract}
Traditional change point detection in dynamic networks assumes abrupt transitions between stationary states, overlooking scenarios of continuous evolution which arise in most real-world applications, such as social networks or physical systems.
We address this gap by formally defining regimes as periods of coherent dynamics in temporal graphs, which we characterize as trajectories along geodesics in a suitably defined graph space.
This original perspective allows us to define regime changes as significant drifts in dynamics, either toward new trajectories or with pace changes.
We leverage graph regression methods to measure the cumulative distance of sequences of observed graphs from the estimated geodesics between their endpoints, in the relevant graph space, which we can combine with change point detection algorithms.
We present experiments on dynamic networks, with changing trajectories and varying speeds, in which we outperform state of the art change point detection models.
Then, we analyse mobility data during the Covid-19 pandemic, and show that our assumptions on regular network evolution lead to change points that are more aligned to external events compared to the outcomes of baseline methods.
Our work is the first to model and detect changes between evolving regimes in graph space, providing a realistic and powerful tool for analyzing complex temporal graph data.
\end{abstract}


\section{Introduction}

\IEEEPARstart{D}{ynamic} graphs are natural data representations in many applications, such as social networks, communication systems, and biological processes, where they describe relationships between entities over time.
Understanding their temporal evolution is necessary to extract knowledge, and a particularly interesting problem consists in identifying temporal segments with consistent behaviors, known as \emph{regimes}, and their changes.
The temporal graph literature \citep{ranshousAnomalyDetectionDynamic2015, zhouSurveyChangePoint2025} addresses this change point detection problem, but existing methods typically assume that the underlying process changes abruptly between stationary models.
This assumption is often violated in practice, as many real-world systems exhibit continuous evolution rather than sudden shifts.

For instance, communities in social networks evolve through changing memberships, mergers, and splits, and identifying the emergence of new groups, or the stabilization of existing ones, can reveal significant shifts in social behaviors.
Communities do not appear or disappear instantaneously but evolve gradually over time.
As a more precise example, let us consider a communication network within a company, where we expect people to exchange more messages with employees in the same department than in others. In case of an internal reorganization merging two teams, people will gradually communicate more within their new peers, but if during this process a second restructuring occurs that transfers some employees to a different unit, the communication network will drift to a new trajectory.
Modeling this kind of scenarios requires models that recognize gradual evolution as a normal behavior, and identify changes between smooth underlying dynamics of the observed system.

\begin{figure}[t]
  \centering
  \begin{subfigure}[c]{\linewidth}
    \centering
    \includegraphics[height=.66\linewidth]{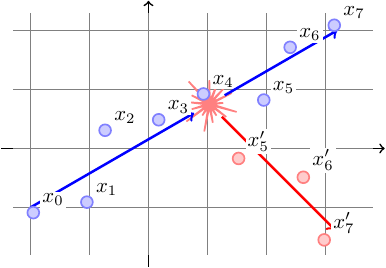}
    \caption{Euclidean space}
    \label{fig:euclidean-regimes}
  \end{subfigure}

  \begin{subfigure}[c]{\linewidth}
    \centering
    \includegraphics[height=.66\linewidth]{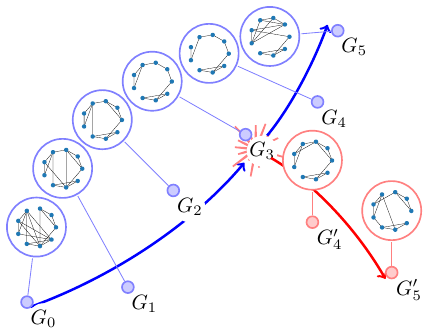}
    \caption{Graph space}
    \label{fig:graph-regimes}
  \end{subfigure}
  \caption{
    Illustration of the parallelism between \emph{coherent evolution} in Euclidean and graph spaces,
    where the underlying regimes correspond respectively to straight lines and \emph{geodesic curves}.
    {\color{blue}Blue arrows} are aligned, and they correspond to a same regime, while the {\color{red}red arrow} shows an alternative regime that arises at change points $t=4$ and $t=3$ respectively.
    Colored points show perturbed observation, in the relative spaces, arising from evolution along the regime if the corresponding color.
  }
  \label{fig:regimes-illustration}
\end{figure}

In this work, we introduce a framework to characterize sequences of graphs with continuous evolution, through regimes that are dynamic.
In contrast to the literature, we explicitly model dynamics within the graph space itself in order to directly capture the properties on which we define coherent dynamics.
We argue that coherent dynamics in graph space should follow two principles:
the evolution proceeds at a constant pace, and consistent steps of the evolution are aligned on the same trajectory.
\zcref{fig:regimes-illustration} illustrates our formalization, drawing the parallel between regular evolution in Euclidean spaces on the one hand (where trajectories with constant speed correspond to straight lines) and consistent evolution in graph spaces on the other hand.
With a proper choice of graph distance, we can formalize such trajectories as \emph{geodesics}, which describe shortest paths in graph space, and which we use to define and identify regimes and their changes.

We formulate the change point detection (CPD) problem as the task of identifying geodesics within the observed sequence and detecting when pace and direction diverge.
To solve this problem, we propose a regression-based framework in graph space.
Since geodesics can be characterized by their endpoints, we assess whether any given graph subsequence aligns with the geodesic connecting its start and end points.
We introduce a \emph{Residual Sum of Squares} (RSS) cost function that quantifies the cumulative squared distances from observed graphs to the geodesic interpolation connecting segment endpoints.
For any segment, low RSS indicates strong alignment with a single regime, while high RSS signals potential regime changes.
This approach bridges geometric analysis and classical change point detection, allowing us to apply established algorithms \cite{truongSelectiveReviewOffline2020} to identify the best segmentation into distinct regimes.
Going back to our previous example, our model would estimate an optimal trajectory from the initial state to the final in the communication network.
If observations deviate significantly from this geodesic, as when organizational restructuring alters the evolution direction, the resulting high RSS indicates that multiple regimes with different trajectories better explain the data.

We describe two practical implementations of our framework, based on different hypotheses on the graph manifold and distance.
Namely, we discuss linear and optimal-transport based graph metrics, exploring their motivation and properties.
To account for the noisy nature of real-world graph sequences, we propose four strategies to sample graphs from continuous geodesics, focusing in particular on discrete sampling.
We discuss their empirical properties and analyze the interplay between the underlying geodesics manifold and the regression distance.

Finally, we present an empirical study to compare change detection methods on continuously evolving graphs.
A first experiment on synthetic data compares performances of multiple CPD methods on evolving Stochastic Block Models (SBMs), with either changes to the trajectory endpoint, or to the evolution speed.
Our total deviation cost consistently outperforms established methods based on graph invariants and prototype dissimilarity.
A second experiment on real-world mobility data from the Covid-19 pandemic in England \citep{panagopoulosTransferGraphNeural2021} provides a qualitative analysis of the detected change points, showing that our methods identify meaningful shifts corresponding to major lockdown events.

We summarize our contributions as follows:
\begin{itemize}
  \item \textbf{Geodesic characterization of graph regimes}.
  We formalize coherent graph dynamics through geodesics in metric spaces, defining \emph{regimes} as periods with consistent temporal evolution and \emph{regime changes} as drifts in dynamics. This formalization offers new modeling perspectives for dynamic networks, beyond the stationary regime assumption prevalent in the literature;

  \item \textbf{Regime identification through regression in graph space}.
  Thanks to our geometric formulation we propose a \emph{Residual Sum of Squares} (RSS) cost that measures deviation of samples from estimated geodesics.
  This provides a theoretically grounded approach to regime detection through regression in graph space, unifying geometric interpretation with statistical inference;

  \item \textbf{Practical graph regression models}.
  We support our framework with practical implementation on three graph distances, namely Frobenius norm, $L_{1,1}$, and Bures-Wasserstein. We discuss strengths and drawbacks of different hypotheses in terms of data fidelity and computational complexity;

  \item \textbf{Analysis of discrete graph sampling from continuous models}.
  To connect our continuous geodesic model with the discrete nature of real-world graph snapshots, we introduce four sampling strategies and analyze empirically how they affect statistical properties and change detection performance;

  \item \textbf{Empirical validation across multiple regime change types}.
  We validate our framework on synthetic data with trajectory shifts and evolution speed changes, and we compare established CPD methods from the literature.
  On real-world data, we demonstrate that geometric-informed models identify structural changes that align with external events, validating the practical relevance of our new framework.
\end{itemize}

\section{Related Work}\label{sec:related}

\noindent Dynamic network analysis has progressively moved from static anomaly scoring to explicit temporal modeling of evolving relational systems. \citet{ranshousAnomalyDetectionDynamic2015} and \citet{zhouSurveyChangePoint2025} provide comprehensive overviews of change point detection on dynamic networks, where dynamic data are typically formulated either as graph snapshot series, with topology evolving over time, or as time series on graphs, where signals evolve on a network. Across these settings, common assumptions include node alignment across time, comparable sampling intervals, and piecewise stationarity or smoothness of the generating process. Under this view, change point detection aims at estimating unknown boundaries that separate distinct temporal behaviors, while regime detection aims at recovering temporally extended segments with internally coherent dynamics.

The modern offline CPD literature \citep{truongSelectiveReviewOffline2020} frames segmentation through three components: a segment cost, a search strategy, and a complexity constraint. Popular solutions combine parametric or non-parametric costs with exact or approximate search algorithms, such as dynamic programming \citep{rigaillPrunedDynamicProgramming2015}, binary segmentation \citep{scottClusterAnalysisMethod1974}, or PELT \citep{killickOptimalDetectionChangepoints2012}, which provides a canonical reference for efficient penalized optimization. Most practical pipelines for dynamic graphs \citep{zhouSurveyChangePoint2025} first map graphs to statistics, embeddings, or dissimilarity vectors and then detect abrupt drifts in that derived space.

The work that comes closest to operating directly in graph space is the study of \citet{zambonConceptDriftAnomaly2018}. They propose a method to detect changes in network stationarity by creating a dissimilarity representation \citep{duin2005dissimilarity} based on the Graph Edit Distance (GED) between samples and a set of prototype graphs. Their framework is conceptually related to ours since it builds upon a native graph distance which can capture structural properties, but authors designed it to detect sudden changes between graph distributions. In contrast, our work is the first to characterize continuously evolving regimes in the graph space and detect changes.

The only previous work to address changes in continuous evolution is from \citet{chenEuclideanMirrorsFirstorder2024}, which tries to identify ``first-order change points'' by analysing change rates between successive steps through one-dimensional representations.
Their approach relies on a \emph{random dot product} graph model, where at each instant $t$ each nodes has a latent position vector which influences edges probability proportionally to the pairwise distance to other nodes.
In this setting, change points in the graph sequence arise from changes in the \emph{latent position process} (LPP), both at the distribution level (zeroth-order) and in the distribution of the increments (first order).
To recognize change points, they define a 1D \emph{Euclidean mirror}, a parameterized line $\phi$ provided by the isomap of the estimated LPP, on which the authors study the rate of change $\norm{\psi(t+1) - \psi(t)}$.
Compared to our paradigm, the iso-mirror setting is insensitive to changes in direction.

A vast body of work tackles anomaly detection on graphs with deep learning \citep{maComprehensiveSurveyGraph2023}. These methods learn latent embeddings of graphs and identify anomalies or change points by detecting drifts in the embedding space.
Unfortunately, the maps from this black-box models deeply transform the geometry of the problem, and evolution properties observed in the latent space do not map directly to graph dynamics, so that properties such as speed of evolution and direction are lost.
For the latter, in particular, the geodesic characterization of regimes that we propose below allows to focus on methods based on interpolation and dissimilarity, which can provide counterfactuals for regime changes.

\section{Geodesic Framework for Detection of Changes in Dynamic-Graph Regimes}\label{sec:framework-definition}

\noindent In this Section we present a formal characterization of coherent graph evolution, on which we define sequences following a geodesic in graph space as \emph{consistent regimes}.
This notion easily extends to \emph{regime changes}, where the underlying trajectory drifts toward a different geodesic.
The \emph{change point detection} (CPD) problem then focuses on identifying these deviations and localizing them in time.
To conclude, we present some practical implementations of our framework based on common graph distances.

\subsection{Graph Distances, Geodesics, and Regimes}\label{ssec:graph-geodesics}

\noindent We study sequences ${\{ (G_t, \tau_t) \}_{t=1}^T}$ of weighted graph snapshots with associated timestamps $\tau_t \in \R$ and~$T$ samples.
Each snapshot $G_t = (V, W_t)$ is defined on a fixed and identifiable set of vertices $V$ of size~$N$.
Each~$G_t$ has an adjacency matrix ${W_t \in \R_+^{N \times N}}$, which indicates the weight~$w_{tnm}$ of the edge at a specific time~$t$ between nodes~$n$ and~$m$, and with ${w_{tnm} = 0}$ indicating its absence.
We focus on undirected graphs, for which~$W^\top = W$.
We suppose timestaps to be ordered and unique, so that $\tau_t < \tau_{t+1}$ for all $t > 0$, and we encompass uniformly sampled snapshots by setting $\tau_t = t \in \mathbb{N}$.

We define \emph{regimes of coherent dynamics} as sequences that satisfy two properties: first, the distance between any sample in a regime is proportional to their distance in time; and second, for each pair of samples, intermediate observations lie close to the shortest path between them.
Therefore, we equip the space $\G$ of graphs with a distance which measures dissimilarity between structures, whose choice reflects what we consider to be \enquote{small changes} between graphs.
In geometry, trajectories that agree to this definition are known as \emph{geodesics}.
Formally, given a metric space $M$ with distance $d$, a \emph{shortest geodesic} is a curve $\gamma: I \mapsto M$ that maps an interval $I \subset \R$ to points in the space such that, for a $v \in \R_+$,
\begin{equation}\label{eq:geodesic-definition}
  d(\gamma(\tau_0), \gamma(\tau_1)) = v \left| \tau_0 - \tau_1 \right|.
\end{equation}

We characterize a graph geodesic as a continuous function ${\Gamma_d: \G \times \G \times [t_i,t_j] \to \G}$ that maps a timestamp~$\tau$ in the interval $[t_i,t_j]$ to a graph $\Gamma_d (G_i, G_j, \tau)$ on the geodesic between the pair of its \emph{endpoint graphs} $(G_{t_i}, G_{t_j})$.
By definition, we have that $\forall \tau \in [t_i,t_j]$ and $G_\tau = \Gamma_d (G_i, G_j, \tau)$,
\begin{equation}\label{eq:distance-endpoints}
  d(G_{t_i}, \Gamma_{\tau}) + d(\Gamma_{\tau}, G_{t_j}) = d(G_{t_i}, G_{t_j})
  .
\end{equation}
We observe that arcs within a geodesic are geodesics themselves, which we define as \emph{aligned}.

With a proper choice of distance, there is a unique shortest geodesic between two points.
We can therefore define a partial ordering by inclusion based on geodesic endpoint timestamps and alignment, which allows to segment continuous trajectories in unique maximal geodesics.
To characterize observed sequences we focus on these maximal geodesics, so that we segment a \emph{consistent regime} as the largest interval over which geodesics are aligned and speed $v$ is constant.

Therefore, given a graph sequence with associated timestamps $\{ (G_t, \tau_t) \}_{t=1}^T$, we say that a subsequence $\{ (G_t, \tau_t) \}_{t=i}^j$ follows a \emph{consistent regime} if it describes or approximates a geodesic, i.e.~$\forall t \in [i,j], G_t \approx \Gamma(G_i, G_j, \tau_t)$ and $\exists v \geq 0$ s.th.~$\forall t, t' \in [i,j], d(G_t, G_{t'}) \approx v | \tau_t - \tau_t' |$.
Samples $t^*$  for which we observe two distinct regimes over $[t_i, t^*]$ and $[t^*, t_j]$ are \emph{change points}, and they become actually the target of unsupervised detection models.
We highlight that with this definition, changes in regimes might describe evolution towards different structures, or changes in pace of a similar dynamic, as $v$ is part of the definition of \zcref{eq:geodesic-definition}.

\subsection{Change Point Detection through Graph Residual Sum of Squares}
\label{ssec:deviation-cost}

\noindent Based on our geodesic interpretation of regimes, we propose to infer them from graph observations by quantifying how well any given subsequence $\{(G_t, \tau_t)\}_{t=i}^j$, adheres to the interpolation between its endpoints.
The idea is to compare the goodness of fit of different segmentations, to find the one that optimally separates the entire sequence into distinct, coherent regimes.

This problem is known as \emph{change point detection} (CPD), and a rich literature proposes multiple ways to approach it.
Most CPD algorithms fall within the characterization of \citet{truongSelectiveReviewOffline2020}, which consists of three blocks: a cost function that measures how well a segment is modeled by a single regime, a regularization term on the number of change points, and a search algorithm to identify the best segmentation of the network sequences.
While the cost function is highly tied to hypotheses on data, the latter two parts consist of model selection methods and are often agnostic to the problem settings.

We propose a \emph{Residual Sum of Squares} (RSS) cost, which measures the cumulative discrepancy of the observed graphs from an estimated underlying geodesic that best fits the subsequence.
The cost for a segment $[i, j]$ is the sum of squared distances from each graph $G_t$ to its projection onto the estimated geodesic~$\hat \Gamma(G_i, G_j, \tau_t)$:
\begin{equation}
  C(i, j) = \sum_{t=i}^j d(G_t, \hat\Gamma(G_i, G_j, \tau_t))^2.
\end{equation}
A low cost indicates that the graphs are well-explained by a single trajectory, while a high cost suggests that at least one regime change has occurred.
Our framework establishes a principled connection between standard statistical methods for points in Euclidean spaces and the geometric analysis of graphs.
In fact, RSS is an established statistical tool to quantify model fit in Ordinary Least Squares regression, but it is unexplored for geodesic estimation in graph space.
This RSS approach is supported by theoretical work on statistical inference properties of the Fréchet mean and variance of Hermitian operators, which comprehend Graph Laplacians, under BW distance \citep{kroshninStatisticalInferenceBures2021}, that we explore later.
In particular, empirical barycenters, and thus interpolation points, are unique and, under certain assumptions, normally distributed.

Our framework is sensitive to two fundamental types of regime changes.
First, to \emph{trajectory shifts} which occur when the system drifts toward a different target structure, corresponding to a change in geodesic endpoints.
For instance, in a social network, this captures a community gradually reorganizing toward a new configuration.
Second, to \emph{evolution speed changes} where the pace of dynamics accelerates or decelerates along a similar structural direction, as the velocity parameter $v$ in \zcref{eq:geodesic-definition} varies.
This might manifest in communication networks where interaction patterns intensify or diminish over time.
Our geodesic-based approach explicitly models these continuous evolutionary properties, enabling the identification of subtle drifts in system dynamics and not only abrupt transitions between stationary states as other traditional change point detection methods.

\subsection{Practical Distance Models}\label{ssec:practical-model}

\begin{figure*}[t]
  \begin{subfigure}[c]{.68\linewidth}
    \centering
    \includegraphics[width=\linewidth]{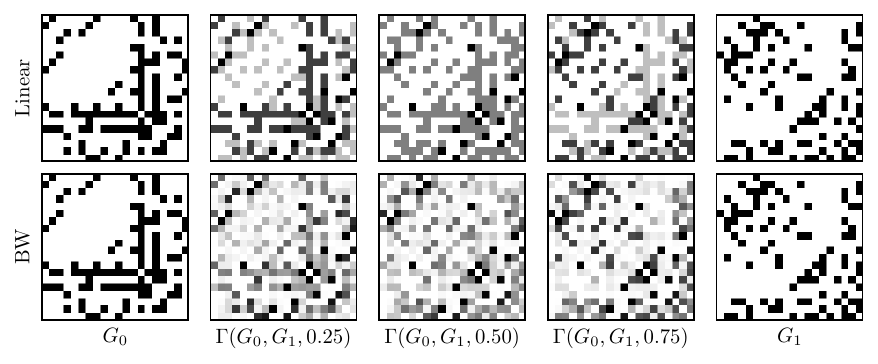}
    \caption{Adjacency matrices}
    \label{fig:geodesic-adj}
  \end{subfigure}
  \hfill
  \begin{subfigure}[c]{.31\linewidth}
    \centering
    \includegraphics[width=\linewidth]{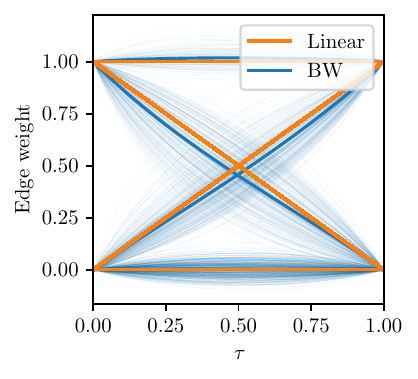}
    \caption{Weight evolution}
    \label{fig:geodesic-weight}
  \end{subfigure}
  \caption{\label{fig:geodesics-ba-sbm}
    Example of graphs on the Linear ($L_{2,2}$) and Bures-Wasserstein geodesics between two graphs $G_0, G_1$, respectively sampled from a Barabasi-Albert model, and from a SBM with two blocks.
    \zcref{fig:geodesic-adj} shows the adjacency matrices for $\tau=0,0.25,0.5,0.75,1$, while \zcref{fig:geodesic-weight} shows the evolution of each edge weight along the geodesics.
    With a linear geodesic all edges change independently and at the same speed ({\color{orange}orange lines}), while by definition of BW trajectories ({\color{blue}blue lines}) edge weights change according to their spectral relevance.
  }
\end{figure*}

\noindent Our formulation allows for flexible modeling through the choice of manifold and distance. In the following, we illustrate the applicability of this framework by focusing on three non-exhaustive distances, for which we can explicitly compute barycenters and which allow us to conveniently characterize geodesics as interpolations between \emph{endpoints}.
\begin{itemize}
  \item Frobenius, or $L_{2,2}$, distance between weight matrices,
  \begin{equation}
    d_{2}(G_i, G_j) = \norm{W_i - W_j}_{F}
    = \sqrt{ \sum_{n,m} \left( w_{inm} - w_{jnm} \right)^2 };
  \end{equation}
  \item $L_{1,1}$ distance between weight matrices,
  \begin{equation}
    d_{1} (G_i, G_j) = \norm{W_i - W_j}_{1,1}
      = \sum_{n,m} \left| w_{inm} - w_{jnm} \right|;
  \end{equation}
  \item Bures-Wasserstein (BW) distance between graph filters \citep{bhatiaBuresWassersteinDistance2019,petricmareticGOTOptimalTransport2019}, which relate the distance between graphs to the distance between associated smooth graph signal distributions.
  In this work we focus on the Laplacians' pseudoinverses~$L^\dagger$, which consider the graph Laplacian as a Gaussian's precision matrix:
  \begin{equation}
      d_{BW}(G_i, G_j)
      = \sqrt{
        \tr \left( L_i^\dagger + L_j^\dagger
        -2 \sqrt{L_{t}^{\dagger/2} L_j L_{t}^{\dagger/2}} \right)
      }.
  \end{equation}
\end{itemize}
We note that the $L_{1,1}$ and $L_{2,2}$ distances correspond to a graph edit distance, as  they measure the cumulative variation on each edge. Since nodes are fixed and identifiable the only possible edit action consists in changing edge weights.
By definition of Euclidean norms, edge differences are considered independently, and therefore these distances cannot take into account global graph structures.

The geodesic with respect to the $L$ distances is the \emph{linear} interpolation between weight matrices $W_i$ and $W_j$, so that for $\tau \in [0,1]$ we have the barycenter
\begin{equation}
  \Gamma_{Lin} (G_i, G_j, \tau) = (1 - \tau) W_i + \tau W_j.
\end{equation}
The Bures-Wasserstein geodesic, studied by \citet{haaslerBuresWassersteinMeansGraphs2024} is of particular interest, as it better preserves spectral and structural properties of the graphs.
A closed-form formula provides the interpolation between two graphs with respect to the BW distance.
The barycenter graph $\Gamma_{BW} (G_i, G_j, \tau)$ has Laplacian $L_{\tau} = S_{\tau}^\dagger$, where $\dagger$ denotes the Moore-Penrose pseudoinverse, and $S_\tau$ is given by
\begin{equation}\label{eq:barycenter-haasler}
  S_\tau = L_i^{1/2} \left(
      (1 - \tau) L_i^\dagger
      + \tau \sqrt{ L_i^{\dagger / 2} L_j^\dagger L_i^{\dagger / 2}}
    \right)^2 L_i^{1/2}.
\end{equation}
Since this formula relies on matrix square roots, it is prone to numerical issues, which might impair iterative algorithms.
Therefore, we provide the following equivalent formulation, which only requires computing the square root of the pseudoinverses product.

\begin{proposition}
  The barycenter graph $\Gamma_{BW} (G_i, G_j, \tau)$ has Laplacian $L_{\tau} = S_{\tau}^\dagger$, where $S_\tau$ is given by
  \begin{equation}\label{eq:barycenter-ours}
    S_\tau = (1 - \tau)^2 L_i^{\dagger} + \tau^2 L_j^{\dagger}
      + \left( \tau - \tau^2 \right) \left(
        \sqrt{ L_i^\dagger L_j^\dagger }
        + \sqrt{ L_j^{\dagger} L_i^\dagger }
      \right)
      .
  \end{equation}
\end{proposition}
\begin{proof}
  Starting from \zcref{eq:barycenter-haasler}, we observe that
  \begin{align}
    S_\tau &= L_i^{1/2} \left(
        (1 - \tau) L_i^\dagger
        + \tau \sqrt{ L_i^{\dagger / 2} L_j^\dagger L_i^{\dagger / 2}}
      \right)^2 L_i^{1/2}
      \nonumber \\
    &= \begin{multlined}[t]
      L_i^{1/2} \left(
        (1 - \tau)^2 L_i^{2\dagger}
        + \tau^2  L_i^{\dagger / 2} L_j^\dagger L_i^{\dagger / 2}
      \right) L_i^{1/2} \\
      + L_i^{1/2} \left(
        \left( \tau - \tau^2 \right)
        L_i^\dagger \sqrt{ L_i^{\dagger / 2} L_j^\dagger L_i^{\dagger / 2}}
      \right) L_i^{1/2} \\
      + L_i^{1/2} \left(
        \left( \tau - \tau^2 \right)
        \sqrt{ L_i^{\dagger / 2} L_j^\dagger L_i^{\dagger / 2}} L_i^\dagger
      \right) L_i^{1/2}
      \end{multlined}
      \label{eq:interpolation-2} \\
    &= \begin{multlined}[t]
      (1 - \tau)^2 L_i^{\dagger}
        + \tau^2 L_j^{\dagger} \\
      + \left( \tau - \tau^2 \right) \left(
        L_i^\dagger L_i^{1/2} \sqrt{
          L_i^{\dagger / 2} L_j^\dagger L_i^{\dagger / 2}
        } L_i^{1/2}
      \right) \\
      + \left( \tau - \tau^2 \right) \left(
        L_i^{1/2} \sqrt{
          L_i^{\dagger / 2} L_j^\dagger L_i^{\dagger / 2}
        } L_i^{1/2} L_i^\dagger
      \right)
      \end{multlined}
      \label{eq:interpolation-3} \\
    &= \begin{multlined}[t]
      (1 - \tau)^2 L_i^{\dagger}
        + \tau^2 L_j^{\dagger} \\
      + \left( \tau - \tau^2 \right) \left(
        L_i^\dagger L_i \sqrt{ L_i^\dagger L_j^\dagger }
        + \sqrt{ L_j^{\dagger} L_i^\dagger } L_i L_i^\dagger
      \right)
      \end{multlined}
      \label{eq:interpolation-4}
  \end{align}
  Between \zcref{eq:interpolation-3} and~\eqref{eq:interpolation-4} we apply the identity $(AB)^{1/2} = A^{1/2} \sqrt{A^{1/2} B A^{1/2}} A^{-1/2}$ from Equations (9) and (10) by \citet{bhatiaBuresWassersteinDistance2019}.
  Finally, we obtain \zcref{eq:barycenter-ours} from~\eqref{eq:interpolation-4} by observing that $L_i L_i^\dagger = Id + \frac{1}{N} \bf{11}^\top$, with the second term being in the nullspace of both $L_i$ and $L_j$ and thus also of their product and its squareroot.
\end{proof}

The BW geodesic serves as an excellent formal model for graph evolution by preserving structural properties, but its computational cost is dominated by the multiple pseudo-inversions and matrix square roots required for interpolating graphs, which scale as $\mathcal O (N^3 T)$ with $N$ number of nodes and $T$ the number of interpolation points.
Conversely, linear interpolation with an $L_{2,2}$, or $L_{1,1}$, norm is more efficient as it only requires elementwise operations on edges of the order $\mathcal O (E T)$, where the number of edges $E$ can be much smaller than $N^2$ in sparse networks.

\zcref{fig:geodesics-ba-sbm} illustrates the Linear and BW geodesics between two graphs $G_0$ and $G_1$, the first sampled from a preferential attachment model, and the second from a Stochastic Block Model (SBM) with two blocks.
In \zcref{fig:geodesic-adj} we see that the linear interpolation appears, by definition, as a uniform fading out of edges from the first graph and a corresponding fading in of edges from the second, so that hubs from $G_0$---identified by dense row and columns---persist longer, and still have a high weight in the midpoint $\Gamma (G_0, G_1, 0.5)$.
On the other hand, the BW geodesic displays a structural evolution, with hub links disappearing faster for nodes that end in different communities.
Looking closer at single weight evolution, \zcref{fig:geodesic-weight} shows that each~$w_{\tau ij}$ aligns to one of four combinations, corresponding to edge $(i,j)$ having a weight of 0 or 1 in~$G_0$ and~$G_1$, namely changing from 0 to 0, 0 to 1, 1 to 0 and 1 to 1.
In the linear interpolation, all edges follow the average lines, while with the BW geodesic each edge describes a different curve, which depends on the change in structural role between the start and end graphs.

\section{Experiments}\label{sec:experiments}

\noindent We propose a range of experiments to compare different change point detection (CPD) models on multiple type of regime changes, and in particular we compare our proposed deviation cost to interpolation costs based on graph representation.
First, we provide precise quantitative results on syntetic data with ground truth regimes; then, we investigate mobility data in England during the Covid-19 epidemics, to qualitatively compare detected change points.

We build upon two established offline CPD algorithms, namely Binary Segmentation (BinSeg) \cite{scottClusterAnalysisMethod1974}, and Pruned Exact Linear Time (PELT) \citep{killickOptimalDetectionChangepoints2012}.
These models aim to minimize a \emph{cost function} by splitting the data into subsequences. They are regularized against false positives with a linear penalty on the number of changes.
We apply our Residual Sum of Squares from graph geodesics (\emph{graph-RSS}) with respect to both BW and linear interpolation on observed networks, and we compare it against multiple costs computed on sequence embeddings, provided by the following graph representations:
\begin{itemize}
  \item \textbf{Graph invariants} \citep{tangAttributeFusionLatent2013}: extract five graph invariants---namely \emph{size, number of triangles, scan,} and \emph{mean and maximum degree}---as representation, and combine them with a multivariate CPD algorithm.
  We extend the original method by computing \emph{global} invariants as well as local ones from six \emph{ego subgraphs}, with randomly sampled centers and either one- or two-hop neighbors.
  \item \textbf{Prototype dissimilarity} \citep{zambonConceptDriftAnomaly2018}: identify a set of graphs in the sequence to use as prototypes, and represent each sample as the vector of distances to each prototype. In the original paper, authors use the \emph{GED} ($L_{1,1}$), but we propose to integrate the \emph{BW distance}.
  \item \textbf{LAD} \citep{huangLaplacianChangePoint2020}: use singular values of the Laplacian as the graph representation and analyse two context windows to compare the graph structure with long and short term behaviors.
  \item \textbf{Iso-mirror} \citep{chenEuclideanMirrorsFirstorder2024}: compute pairwise graph distances over procrustes-aligned spectral embeddings, then compute CDMS (Classical Multidimensional Scaling) and ISOMAP \citep{tenenbaumGlobalGeometricFramework2000}.
\end{itemize}
For the baseline models, we treat the cost selection as a hyperparameter, choosing the best performing one between \emph{piecewise linear interpolation}, \emph{least squares RSS}, and \emph{autoregressive model change} \citep{baiVectorAutoregressiveModels2000}.

We perform all experiments on a 32-cores CPU computer, and we provide the implementation of our original methods, and the source code for reproducing experiments as a git repository\footnote{https://github.com/LTS4/dynamic-graph-regimes}. For baselines costs and change algorithms we use the implementations from the Ruptures library \citep{truongSelectiveReviewOffline2020}, and for prototype dissimilarity we use the implementation from \citet{zambonConceptDriftAnomaly2018}.


\subsection{Change Point Detection on Synthetic Data}

\noindent Our primary quantitative evaluation relies on synthetic data. This is a necessary choice because, to the best of our knowledge, no real-world datasets exist with ground-truth annotations of continuous evolutionary regimes. By generating data from known geodesic models, we can rigorously assess the ability of each method to recover the true underlying dynamics and change points.

\begin{table*}[t]
  \centering
  \caption{\label{tab:sbm-speed}
    Rand Index (\%) values for CPD on synthetic SBM with \emph{pace changes} on fixed trajectories.
    \textit{Our methods} are italic, \textbf{best scores} are bold, and \underline{second best} underlined.
    Higher scores represent better agreement between ground truth and detected regimes.
    We report median values from 9 experimental settings.
  }
  \sisetup{detect-weight, mode=text}
\begin{tabular}{l|SSSSSSSS}
\toprule
Geodesic          & {bw}       & {bw}       & {bw}        & {bw}      & {linear}   & {linear}   & {linear} \\
Sampling strategy & {additive} & {thr edge} & {thr fixed} & {thr iid} & {additive} & {thr edge} & {thr iid} \\
\midrule
Graph invariants             & 76.06 & \underline{74.91} & \bfseries 77.78 & 73.25 & 78.72 & 75.09 & 75.38 \\
LAD                          & 65.32 & 66.56 & 72.39 & 71.15 & 63.87 & 70.33 & 70.64 \\
Iso-mirror                   & 66.93 & 68.45 & 72.88 & 68.66 & 68.55 & 69.13 & 67.53 \\
Prototype dis.~$L_{1,1}$     & 79.66 & \bfseries 80.47 & \underline{76.76} & \bfseries 75.57 & \bfseries 80.36 & \bfseries 80.05 & \bfseries 76.22 \\
\itshape Prototype dis.~BW   & \underline{79.86} & 77.30 & 75.70 & \underline{75.35} & 79.57 & \underline{77.27} & \underline{75.48} \\
\itshape graph-RSS $L_{2,2}$ & \bfseries 86.19 & 67.97 & 76.49 & 55.56 & \underline{80.15} & 68.43 & 52.54 \\
\itshape graph-RSS BW        & 73.68 & 77.30 & 75.87 & 55.26 & 79.72 & 76.47 & 54.36 \\
\bottomrule
\end{tabular}

\end{table*}

\begin{table*}[t]
  \caption{\label{tab:sbm-synth}
    Rand Index (\%) values for CPD on synthetic SBM with \emph{trajectory changes}.
    \textit{Our methods} are italic, \textbf{best scores} are bold, and \underline{second best} underlined.
    We report median values from 9 experimental settings.
 }
  \begin{subtable}{\linewidth}
    \centering
    \caption{\label{tab:synthetic-sbm-full-rand} Full evolution to endpoint graph.}
    \sisetup{detect-weight, mode=text}
\begin{tabular}{l|SSSSSSSS}
\toprule
Geodesic          & {bw}       & {bw}       & {bw}        & {bw}      & {linear}   & {linear}   & {linear} \\
Sampling strategy & {additive} & {thr edge} & {thr fixed} & {thr iid} & {additive} & {thr edge} & {thr iid} \\
\midrule
Graph invariants             & 68.40 & 69.82 & 73.51 & 76.91 & 67.16 & 68.76 & 86.21 \\
LAD                          & \underline{77.16} & 62.59 & 60.87 & 60.64 & 73.88 & 63.23 & 59.34 \\
Iso-mirror                   & 60.84 & 66.42 & 64.06 & 62.87 & 63.17 & 62.05 & 63.07 \\
Prototype dis.~$L_{1,1}$     & 74.88 & \bfseries 76.80 & \underline{89.69} & \underline{90.88} & 72.86 & \bfseries 76.75 & \underline{92.09} \\
\itshape Prototype dis.~BW   & 76.39 & 73.88 & 89.16 & 87.05 & 78.12 & 75.43 & 88.09 \\
\itshape graph-RSS $L_{2,2}$ & 76.72 & 59.35 & 62.87 & 90.86 & \bfseries 93.51 & 59.35 & 91.85 \\
\itshape graph-RSS BW        & \bfseries 92.22 & \underline{75.20} & \bfseries 89.93 & \bfseries 92.69 & \underline{90.01} & \underline{75.70} & \bfseries 94.11 \\
\bottomrule
\end{tabular}

  \end{subtable}
  \newline

  \begin{subtable}{\linewidth}
    \centering
    \caption{\label{tab:synthetic-sbm-partial-rand} Change with possibly partial evolution.}
    \sisetup{detect-weight, mode=text}
\begin{tabular}{l|SSSSSSSS}
\toprule
Geodesic          & {bw}       & {bw}       & {bw}        & {bw}      & {linear}   & {linear}   & {linear} \\
Sampling strategy & {additive} & {thr edge} & {thr fixed} & {thr iid} & {additive} & {thr edge} & {thr iid} \\
\midrule
Graph invariants             & 67.81 & 68.76 & 73.04 & 74.58 & 74.12 & 69.32 & 76.01 \\
LAD                          & 68.33 & 66.39 & 64.82 & 63.01 & 68.48 & 64.10 & 62.30 \\
Iso-mirror                   & 63.07 & 62.09 & 64.41 & 58.66 & 64.90 & 63.20 & 63.52 \\
Prototype dis.~$L_{1,1}$     & 75.10 & \underline{77.38} & \underline{73.20} & \underline{81.99} & 75.24 & \underline{76.91} & \underline{82.27} \\
\itshape Prototype dis.~BW   & 75.03 & 72.87 & 72.25 & 79.52 & 74.97 & 74.45 & 79.21 \\
\itshape graph-RSS $L_{2,2}$ & \underline{89.71} & 60.23 & 66.71 & 63.61 & \bfseries 94.07 & 60.39 & 65.09 \\
\itshape graph-RSS BW        & \bfseries 91.93 & \bfseries 78.43 & \bfseries 73.94 & \bfseries 92.27 & \underline{90.30} & \bfseries 79.80 & \bfseries 92.53 \\
\bottomrule
\end{tabular}

  \end{subtable}
\end{table*}

\subsubsection{Sampling Along Trajectories}\label{sec:sampling}

\begin{figure}[t]
  \centering
  \includegraphics[width=\linewidth]{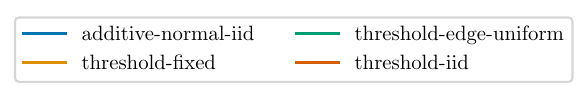}
  \begin{subfigure}[t]{\linewidth}
    \centering
    \includegraphics[height=.45\linewidth]{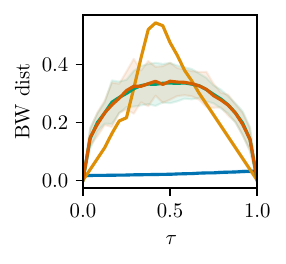}
    \hfill
    \includegraphics[height=.45\linewidth]{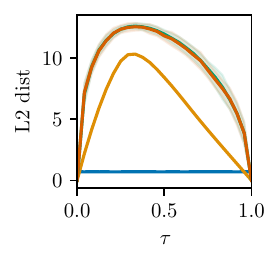}
    \caption{Distances to BW geodesic}
  \end{subfigure}
  \begin{subfigure}[t]{\linewidth}
    \centering
    \includegraphics[height=.45\linewidth]{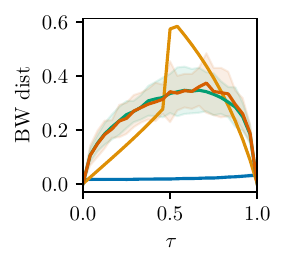}
    \hfill
    \includegraphics[height=.45\linewidth]{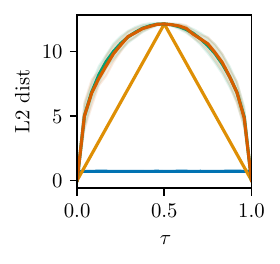}
    \caption{Distances to Linear geodesic}
  \end{subfigure}
  \caption{\label{fig:sample-geodesic-distance}
    Mean and std of distances of discrete samples from the continuous geodesic between a SBM with two blocks ($\tau = 0$) to a SBM with three blocks ($\tau=1$).
  }
\end{figure}

In real-world scenarios, we generally observe noisy realizations of an underlying process, due to measurement errors and overall uncertainty.
A peculiar problem of networks modeling complex systems arises in observing discrete interactions which might depend on continuous hidden variables.
For instance, considering a social network, we can suppose that the underlying strength of a relationship between two individuals is a continuous variable, but our measurements are often discrete with edges representing an explicit interaction, such as an exchanged message.

To bridge this gap, we analyse four methods to sample graphs along continuous geodesics, which allow to generate realistic data according to different hypotheses.
We propose the following strategies, of which the \emph{thresholding} ones produce discrete graphs:
\begin{itemize}
  \item Additive noise (\emph{additive}): $w_{tnm} + \epsilon_{tnm}$ with iid $\epsilon_{tnm} \sim \mathcal N (0, \sigma)$;
  \item Fixed threshold (\emph{thr-fix}): $w_{tnm} > \gamma$, all edges share a predefined threshold;
  \item Threshold by edge (\emph{thr-edge}): $w_{tnm} > \gamma_{nm}$ with $\gamma_{nm} \sim U(0,1)$, where each edge has an independent threshold fixed over time sampled from a uniform distribution;
  \item Independent thresholding (\emph{thr-iid}): $w_{tnm} > \gamma_{tnm}$ with $\gamma_{tnm} \sim U(0,1)$, which corresponds to independently sampling each edge from a Bernoulli distribution with probability equal to the edge weight $w_{tnm}$.
\end{itemize}
We highlight that the iid additive perturbation has an overall Gaussian behavior only within Euclidean spaces, and thus with $L$ norms, but this does not hold for the BW distance.
Using a \emph{fixed} threshold for all edges is only interesting with the BW interpolation, as edge weights change with different paces, as we see in \zcref{fig:geodesic-weight}. With a linear geodesic, then all edge weights change at the same pace, and thus they appear or disappear in blocks.
Having an \emph{edge} threshold that is constant over time provides a noisy evolution where addition and deletion are, in general, monotonic.
Finally, the \emph{independent} sampling strategy is the noisiest setting, as edge sampling is independent over time and thus, the further $w_{tnm}$ is from 0 or 1, the higher the edge variance.

\zcref{fig:sample-geodesic-distance} reports the mean and variance of distances between discrete samples and the underlying geodesic as a function of the interpolation coefficient $\tau$ between SBM samples with two and three blocks.
Across random thresholding strategies, deviations follow a roughly parabolic profile peaking near the midpoint, and as expected the variance grows with temporal separation from the endpoints.
Designing a discrete sampling scheme that yields homoscedastic distances along the trajectory remains an open problem.
For fixed thresholds, BW distances exhibit sharp jumps, reflecting abrupt structural transitions, and this effect is stronger when distances are measured to the linear geodesic.
Conversely, $L_{2,2}$ distances to the BW geodesic are skewed because constant structural change in BW induces non-linear edgewise evolution, as already highlighted in \zcref{fig:geodesic-weight}.

For $L$ distances with linear interpolation, the sampled adjacency entries are independent Bernoulli variables with parameter $w_{tnm}$, so the expected squared deviation is a sum of variances.
This yields an upper bound proportional to $\sum_{n,m} w_{tnm} (1 - w_{tnm})$, which is maximized when ${w_{tnm} = 0.5}$; hence the largest deviation occurs near the midpoint of the trajectory, where changing edges concentrate around $0.5$.
For BW geodesics, edge weights evolve at different rates, so there is no single $\tau$ that maximizes variance across edges.
Nevertheless, since BW trajectories remain within the convex hull of edge weights, the variance remains bounded by the linear case, and so does the expected deviation.
To the best of our knowledge, no closed-form variance analysis exists for such sampling methods under BW distances, so we rely on the empirical trends shown in \zcref{fig:sample-geodesic-distance}.

\begin{table*}[!t]
  \centering
  \caption{\label{tab:covid-events}
    Covid-19 events in England between March 3rd and May 12th, and detected change points on daily mobility graphs between England NUTS3 regions on the same period.
    Segmentation costs are optimized for four changes with exact search.
    We list change dates in the row corresponding to the closest real world event.
  }
  \begin{tabular}{p{40ex}l|ccccccc}
\toprule
 & & LAD & Iso-mirror & \multicolumn{2}{c}{Prototype dis.} & \multicolumn{2}{c}{\itshape Graph RSS} \\
Events                                  & Date       &        &        & $L_{1,1}$ & BW  & $L_{2,2}$ & BW  \\
\midrule
First \enquote{official} COVID-19 death & 05 Mar     &   --   &   --   & 08 Mar &   --   & 06 Mar &   --   \\
Cheltenham Festival (150k visitors)     & 10--13 Mar &   --   & 11 Mar & 11 Mar & 12 Mar & 11 Mar &   --   \\
PM advises against contact and travel   & 16 Mar     & 18 Mar &   --   &   --   &   --   & 18 Mar & 18 Mar \\
Closure of cafes, pubs and restaurants  & 20 Mar     & 20 Mar &   --   &   --   & 22 Mar &   --   & 22 Mar \\
Lockdown starts                         & 26 Mar     & 25 Mar & 25 Mar & 31 Mar &   --   &   --   &   --   \\
Maximum daily hospital admissions       & 02 Apr     & 03 Apr & 04 Apr & 04 Apr & 03 Apr & 31 Mar & 06 Apr \\
Easter day                              & 12 Apr     &   --   &   --   &   --   & 15 Apr &   --   & 10 Apr \\
--                                      &   --       &   --   & 28 Apr &   --   &   --   &   --   &   --   \\
\bottomrule
\end{tabular}

\end{table*}

\subsubsection{Detection Performances}

We study the CPD performance of multiple models on synthetic datasets with regime changes correspoding to the characterization from \zcref{ssec:graph-geodesics}.
Each dataset contains multiple sequences of 50 observations, with each timestamp having a~5\% probability of being a change point.
We separately study three kinds of regime changes, namely: change of evolution pace between fixed endpoints, trajectory change after full evolution to endpoint, and change of endpoint after partial evolution.
Each sequence has graphs with a fixed number of nodes, either 50, 100 or 200, and their endpoint graphs are independently sampled from SBMs with random block sizes and average degree set to 0.1, 0.2, or 0.3, for a total of 9 sequence types.
We compute geodesics on BW and linear models, and sample discrete graphs along the geodesics with the strategies described in \zcref{sec:sampling}.
CPD algorithms predict change points to segment sequences from each of the 9 above combinations of graph sizes and average degrees; we select hyperparameters on a training set of $10 \times 9 = 90$ sequences, and we test models on $50 \times 9 = 450$ independent sequences.

\zcref{tab:sbm-synth,tab:sbm-speed} show model performances on the three regime change types for different combinations of underlying geodesics and sampling strategies in terms of Rand Index, which measures the agreement between predicted and true change points, with values ranging from 0 to 1 and higher scores representing better alignement.
We also report Haussdorf distances and F1 scores in \zcref{sec:additional-metrics}.

For the pace change experiment in \zcref{tab:sbm-speed} we see that the prototype dissimilarity models generally outperform other baselines, even though they lag behind our \textit{graph-RSS $L_{2,2}$} with additive edge noise on the BW geodesic, and are on par with it on the linear one.
This result is not surprising, as the prototype model aligns well with the constant speed hypothesis. In fact, through a piecewise-linear regression of the distances from points on the geodesic, i.e. the prototypes, one monitors the speed on the same.
The three graph embedding models generally perform worse than the distance-based ones, with the notable exception of the low performances of both \emph{graph-RSS} methods on samples from independent edge tresholding.
We suppose that the sampling variance of those sequences is higher than the explained variance from graph interpolation, thus hindering the RSS hypothesis.

Analysing the results on full trajectory changes from \zcref{tab:sbm-synth}, we see that our proposed methods outperform the baselines across most settings.
In particular, the residual sum of squares on the BW distance between embeddings (\textit{graph-RSS BW}) achieves the best performance in ten out of fourteen scenarios, and the second best in all others.
The \textit{graph-RSS $L_{2,2}$} has strong performances on additive noise settings, but struggles with uniform edge thresholding.
The \emph{Prototype dissimilarity} methods are strong contenders, with the $L_{1,1}$ declination showing competitive results.
This highlights that distance-based methods are more robust and effective for change point detection in dynamic graphs, and our RSS approach is particularly interesting for partial evolution, as \zcref{tab:synthetic-sbm-partial-rand} shows a significant improvement over the second performing models.


\subsection{England Covid-19 Mobility}

\noindent We propose a novel study on the \enquote{England Covid Dataset} from \citet{panagopoulosTransferGraphNeural2021}, and distributed with PyTorch Geometric Temporal\footnote{\href{https://pytorch-geometric-temporal.readthedocs.io/en/0.56.2/index.html}{Version 0.56.2}}.
This experiment focuses on a qualitative comparison of change points detected by different algorithms.
The dataset provides daily mobility graphs of people moving between~129 counties in England---according to the Nomenclature of Territorial Units for Statistics (NUTS3)---from March 3rd to May 12th, based on Facebook Data For Good disease prevention maps.
Each graph has weighted edges counting the daily number of people travelling from one region to another one.
We compare LAD, iso-mirror, prototype dissimilarity and total variation costs with both linear and BW distances.
For all methods, we look for the best segmentation into five regimes, separated by four change points, of the sixty days of measurements.
We perform segmentation with an exact Dynamic Programming search algorithm.

\zcref{tab:covid-events} shows the dates identified by different methods, along with major events in the Covid-19 pandemic\footnotemark.
The most relevant event affecting mobility is the start of the lockdown, with the first closures on March 20th and full enforcement on March 26th.
Both dates are recovered by the LAD algorithm, and regression on iso-mirror embeddings identifies the latter date.
Models based on BW distance indentify a single regime change within those dates, two days after bars and restaurant closes, which suggest that they might group the restriction effects within a same transition dynamic.
The BW metric seems more suited for detecting structural changes in mobility patterns than $L_{1,1}$, which has no regime separation on that period.
Both graph-RSS models and LAD highlight a drift arount March 18, which aligns with a national statement from the Prime Minister which advised against \enquote{non-essential} contact and travel.
Four models identify a change around March 11, which corresponds to the Chetenham festival which some sources identify as a major spreading event, and which we might consider as a mobility anomaly, with 150k people attending from all-over the Country.
Finally, all models agree on a regime change in the first days of April, when hospital admissions peaked.

\footnotetext{
  \enquote{Timeline of the COVID-19 pandemic in the United Kingdom (January--June 2020)}, \href{https://w.wiki/FJWj}{Wikipedia}; and
  \enquote{Healthcare in England: patients in hospital}, \href{https://web.archive.org/web/20210730154439/https://coronavirus.data.gov.uk/details/healthcare?areaType=nation&areaName=England}{archived GOV.UK}.
}

\section{Discussion and Open Directions}\label{sec:discussion}

\noindent  Regarding the choice of geodesic and distance metric for CPD as described in \zcref{ssec:graph-geodesics}, one should consider the fundamental trade-off between model fidelity and computational cost.
Search algorithms for CPD compute the cost for any segmentation candidate, and their complexity with respect to the number of samples~\(T\) ranges between \(\mathcal O (T \log T)\) for binary segmentation and \(\mathcal O (T^2)\) for PELT.
In this work we focused on the implementations from \zcref{ssec:practical-model}, based on the Bures-Wasserstein and Euclidean distances.
Still, the computational complexity of BW distance and interpolation makes them impractical for large graphs with many search algorithms, with a cumulative complexity of \(\mathcal O (N^3 T^2 \log T)\) for binary segmentation.
In comparison, CPD based on linear interpolation and \(L\) distances would run in \(\mathcal O (E T^2 \log T)\), which scales better with graph size and allows for sparse computing, at the expense of overseeing structural properties.

Testing our geodesic framework on graph sequences with varying node sets is an interesting direction for future research.
This more general problem requires distances that do not assume node correspondence, such as the Gromov-Wasserstein (GW) distance \citep{titouanOptimalTransportStructured2019} or Graph Edit Distance \citep{bunkeInexactGraphMatching1983}. These metrics compare graphs of potentially different sizes by finding an optimal alignment or a minimal set of edits. While conceptually powerful, their practical application presents significant hurdles. Both GW and GED are computationally expensive (with GED being NP-hard in the general case), which makes them difficult to deploy in search algorithms, as they require numerous cost evaluations. While these distances offer a path toward a more general theory, developing computationally tractable approximations and robust geodesic regression methods for them is outside the scope our work.

Another important extension of our framework lies in online regime change detection.
The current approach builds on the standard offline detection pipeline, which consists of applying a search algorithm to identify the segmentation with the lowest cost and the least number of changes, among all possible ones on available data.
This method can be applied to online data, either by re-running them at certain intervals on all collected data, or on the expanding window starting from the last change point and growing on incoming observations.
However, it would be an exciting opportunity to leverage our definition of consistent evolution through geodesics to forecast trajectories, and directly estimate the deviation of incoming graphs from the expected evolution.

\section{Conclusion}

\noindent In this work, we moved beyond the traditional assumption of abrupt changes between stationary states to analyze continuously evolving dynamic graphs. Our primary contribution is the introduction of a formal framework that characterizes graph evolution in terms of trajectories in a graph metric space. We define coherent \emph{regimes} as periods of stable evolution along a single geodesic path and \emph{regime changes} as deviations from this trajectory. This geometric perspective provides a principled way to understand graph dynamics, accounting for not only abrupt shifts but also gradual drifts in the system's evolution.

Based on this formalism, we developed a graph \emph{residual sum of squares cost}, a practical method for quantifying how well a sequence of graphs adheres to an underlying geodesic. This allows us to use effective algorithms to perform change point detection by solving a geodesic regression problem. To bridge the gap between our continuous model and the discrete nature of real-world observations, we also proposed several sampling strategies to generate realistic graph sequences from these geodesics.

Our experimental evaluation validated this framework on both synthetic and real-world data. On synthetic sequences generated from SBMs, our total deviation cost achieved strong performance in identifying change points. More importantly, on a real-world dataset of mobility in England during the COVID-19 pandemic, our method successfully identified change points that aligned with major real-world events. We notably observed that models using the Bures-Wasserstein distance were particularly effective at detecting the structural changes corresponding to the national lockdown.

Future work could explore other graph distances, develop more sophisticated geodesic regression techniques---which might incorporate node and edge features, and exogenous variables---and apply this framework to a wider range of real-world problems.



\printbibliography


\appendix

\subsection{Additional Results for CPD on Synthetic Data}
\label{sec:additional-metrics}

\noindent \zcref{tab:speed-sbm-additional, tab:synthetic-sbm-additional, tab:synthetic-sbm-partial-additional} compare change detection metrics for multiple algorithms on synthetic SBM trajectories, with \textit{our methods} in italic. \textbf{Best scores} are bold, and \underline{second best} underlined. Results are median values over 9 different experimental settings.

The Haussdorf distance is the highest difference in time from a detected changepoint to the closest ground truth one.
The F1 score is computed by defining the CPD problem as a classification one.
More precisely, we define true positives as predicted change points within a small distance from ground truth changes, and false positives and negatives as predicted, or ground truth, changes without corresponding close labels.

\begin{table*}[ht]
  \centering
  \caption{\label{tab:speed-sbm-additional}
   Additional metrics on CPD with changing speed.
  }
  \begin{subtable}{\linewidth}
    \centering
    \caption{\label{tab:speed-sbm-hauss}
      Haussdorf distance.
    }
    \sisetup{detect-weight, mode=text}
\begin{tabular}{l|SSSSSSSS}
\toprule
Geodesic          & {bw}       & {bw}       & {bw}        & {bw}      & {linear}   & {linear}   & {linear} \\
Sampling strategy & {additive} & {thr edge} & {thr fixed} & {thr iid} & {additive} & {thr edge} & {thr iid} \\
\midrule
Graph invariants             & 11.83 & 12.68 & \underline{11.03} & 14.19 & 10.65 & 12.26 & \underline{12.53} \\
LAD                          & 16.68 & 19.60 & 16.85 & 16.69 & 20.66 & 17.19 & 15.45 \\
Iso-mirror                   & 18.45 & 17.95 & 12.96 & 17.25 & 17.73 & 17.60 & 18.31 \\
Prototype dis.~$L_{1,1}$     & 9.45 & \bfseries 9.26 & 12.28 & \bfseries 13.09 & \bfseries 9.53 & \bfseries 9.54 & \bfseries 12.37 \\
\itshape Prototype dis.~BW   & \underline{9.37} & 11.39 & 13.09 & \underline{14.09} & \underline{10.13} & 11.50 & 13.58 \\
\itshape graph-RSS $L_{2,2}$ & \bfseries 6.82 & 17.83 & \bfseries 10.52 & 29.59 & 13.30 & 17.10 & 33.04 \\
\itshape graph-RSS BW        & 17.31 & \underline{10.82} & 12.78 & 31.13 & 10.73 & \underline{10.81} & 30.53 \\
\bottomrule
\end{tabular}

  \end{subtable}\newline

  \begin{subtable}{\linewidth}
    \centering
    \caption{\label{tab:speed-sbm-f1}
      F1 score.
    }
    \sisetup{detect-weight, mode=text}
\begin{tabular}{l|SSSSSSSS}
\toprule
Geodesic          & {bw}       & {bw}       & {bw}        & {bw}      & {linear}   & {linear}   & {linear} \\
Sampling strategy & {additive} & {thr edge} & {thr fixed} & {thr iid} & {additive} & {thr edge} & {thr iid} \\
\midrule
Graph invariants           & 68.84 & 69.05 & 59.12 & \underline{64.72} & 72.22 & 68.69 & \bfseries 66.43 \\
LAD                        & 57.31 & 53.89 & 51.41 & 48.04 & 53.48 & 57.51 & 46.93 \\
Iso-mirror                  & 53.92 & 53.42 & \bfseries 61.35 & 52.75 & 55.24 & 54.94 & 54.37 \\
Prototype dis.~$L_{1,1}$   & \underline{75.40} & \underline{74.49} & 54.88 & 61.92 & \bfseries 74.94 & \underline{73.64} & 62.07 \\
\itshape Prototype dis.~BW & 74.82 & 70.62 & 53.82 & \bfseries 64.93 & \underline{74.56} & 67.64 & \underline{65.89} \\
\itshape graph-RSS $L_{2,2}$ & \bfseries 78.83 & 42.56 & 54.39 & 34.13 & 70.00 & 43.55 & 31.58 \\
\itshape graph-RSS BW        & 65.00 & \bfseries 77.54 & \underline{59.81} & 38.46 & 73.53 & \bfseries 75.95 & 38.29 \\
\bottomrule
\end{tabular}

  \end{subtable}
\end{table*}

\begin{table*}[ht]
  \centering
  \caption{\label{tab:synthetic-sbm-additional}
    Additional metrics on CPD with endpoint deviation and full evolution.
  }
  \begin{subtable}\linewidth
    \centering
    \caption{\label{tab:synthetic-sbm-hauss}
      Haussdorf distance.
    }
    \sisetup{detect-weight, mode=text}
\begin{tabular}{l|SSSSSSSS}
\toprule
Geodesic          & {bw}       & {bw}       & {bw}        & {bw}      & {linear}   & {linear}   & {linear} \\
Sampling strategy & {additive} & {thr edge} & {thr fixed} & {thr iid} & {additive} & {thr edge} & {thr iid} \\
\midrule
Graph invariants           & 20.46 & 17.74 & 16.72 & 15.69 & 19.39 & 17.57 & 12.72 \\
LAD                        & 17.92 & 25.51 & 21.35 & 23.75 & 19.85 & 25.97 & 24.53 \\
Iso-mirror                 & 24.11 & 27.36 & 24.04 & 21.70 & 22.95 & 21.74 & 24.73 \\
Prototype dis.~$L_{1,1}$   & 14.32 & \bfseries 13.14 &  \bfseries 9.91 & \underline{10.29} & 14.67 & \bfseries 12.99 &  \underline{9.70} \\
\itshape Prototype dis.~BW & \underline{13.79} & 14.59 & 10.85 & \bfseries 10.27 & 13.10 & \underline{14.21} & 10.76 \\
\itshape graph-RSS $L_{2,2}$    & 16.23 & 22.90 & 20.62 & 11.95 & \bfseries 10.00 & 22.90 & 11.28 \\
\itshape graph-RSS BW           & \bfseries 10.52 & \underline{15.33} & \underline{10.25} & 10.51 & \underline{11.05} & 14.62 &  \bfseries 9.40 \\
\bottomrule
\end{tabular}

  \end{subtable}\newline

  \begin{subtable}{\linewidth}
    \centering
    \caption{\label{tab:synthetic-sbm-f1}
      F1 score.
    }
    \sisetup{detect-weight, mode=text}
\begin{tabular}{l|SSSSSSSS}
\toprule
Geodesic          & {bw}       & {bw}       & {bw}        & {bw}      & {linear}   & {linear}   & {linear} \\
Sampling strategy & {additive} & {thr edge} & {thr fixed} & {thr iid} & {additive} & {thr edge} & {thr iid} \\
\midrule
Graph invariants           & 52.40 & 60.35 & 62.10 & 69.28 & 54.66 & 59.70 & 79.35 \\
LAD                        & 59.63 & 44.64 & 44.37 & 38.71 & 57.68 & 47.76 & 37.50 \\
Iso-mirror                 & 46.20 & 42.08 & 52.86 & 47.59 & 47.43 & 44.10 & 46.87 \\
Prototype dis.~$L_{1,1}$   & 69.65 & \bfseries 75.14 & \bfseries 88.41 & \bfseries 85.87 & 67.77 & \bfseries 74.45 & \underline{85.72} \\
\itshape Prototype dis.~BW & \underline{75.15} & 68.48 & \underline{86.27} & \underline{85.31} & \underline{78.73} & 69.44 & \bfseries 86.02 \\
\itshape graph-RSS $L_{2,2}$    & 65.24 & 34.07 & 41.33 & 59.52 & 76.78 & 34.07 & 62.52 \\
\itshape graph-RSS BW           & \bfseries 79.80 & \underline{73.03} & 84.19 & 82.03 & \bfseries 79.58 & \underline{73.94} & 81.55 \\
\bottomrule
\end{tabular}

  \end{subtable}
\end{table*}

\begin{table*}[ht]
  \centering
  \caption{\label{tab:synthetic-sbm-partial-additional}
    Additional metrics on CPD with endpoint deviation and partial evolution.
  }
  \begin{subtable}\linewidth
    \centering
    \caption{\label{tab:synthetic-sbm-partial-hauss}
      Haussdorf distance.
    }
    \sisetup{detect-weight, mode=text}
\begin{tabular}{l|SSSSSSSS}
\toprule
Geodesic          & {bw}       & {bw}       & {bw}        & {bw}      & {linear}   & {linear}   & {linear} \\
Sampling strategy & {additive} & {thr edge} & {thr fixed} & {thr iid} & {additive} & {thr edge} & {thr iid} \\
\midrule
Graph invariants           & 19.05 & 17.85 & 18.08 & 16.74 & 16.72 & 17.11 & 16.92 \\
LAD                        & 20.02 & 22.76 & 25.70 & 25.49 & 19.93 & 23.48 & 23.15 \\
Iso-mirror                 & 23.45 & 26.46 & 22.47 & 29.23 & 21.56 & 22.51 & 22.12 \\
Prototype dis.~$L_{1,1}$   & 14.84 & \bfseries 13.34 & \bfseries 16.89 & \underline{13.52} & 14.83 & \bfseries 13.58 & \underline{13.17} \\
\itshape Prototype dis.~BW & 14.80 & 15.43 & 18.01 & 14.56 & 14.89 & 14.75 & 14.65 \\
\itshape graph-RSS $L_{2,2}$    & \underline{11.06} & 22.70 & 21.17 & 21.26 &  \bfseries 9.46 & 22.70 & 20.70 \\
\itshape graph-RSS BW           &  \bfseries 9.79 & \underline{13.58} & \underline{17.14} &  \bfseries 9.25 & \underline{10.50} & \underline{13.89} &  \bfseries 8.53 \\
\bottomrule
\end{tabular}

  \end{subtable}\newline

  \begin{subtable}{\linewidth}
    \centering
    \caption{\label{tab:synthetic-sbm-partial-f1}
      F1 score.
    }
    \sisetup{detect-weight, mode=text}
\begin{tabular}{l|SSSSSSSS}
\toprule
Geodesic          & {bw}       & {bw}       & {bw}        & {bw}      & {linear}   & {linear}   & {linear} \\
Sampling strategy & {additive} & {thr edge} & {thr fixed} & {thr iid} & {additive} & {thr edge} & {thr iid} \\
\midrule
Graph invariants           & 55.53 & 57.28 & 58.15 & 65.74 & 60.87 & 60.11 & 67.31 \\
LAD                        & 49.07 & 51.78 & 47.09 & 41.47 & 51.17 & 50.42 & 39.60 \\
Iso-mirror                 & 45.26 & 44.40 & 49.13 & 40.83 & 49.30 & 46.58 & 49.04 \\
Prototype dis.~$L_{1,1}$   & 69.57 & \bfseries 74.41 & 62.15 & \underline{76.09} & 69.80 & \bfseries 74.04 & \underline{79.05} \\
\itshape Prototype dis.~BW & 70.65 & 67.73 & \underline{62.66} & 75.13 & 70.87 & 70.00 & 75.82 \\
\itshape graph-RSS $L_{2,2}$    & \underline{74.68} & 35.46 & 42.60 & 36.19 & \underline{77.02} & 35.73 & 37.13 \\
\itshape graph-RSS BW           & \bfseries 81.33 & \underline{74.05} & \bfseries 69.83 & \bfseries 90.68 & \bfseries 80.85 & \bfseries 74.04 & \bfseries 95.03 \\
\bottomrule
\end{tabular}

  \end{subtable}
\end{table*}

\end{document}